\documentclass{article}

     \PassOptionsToPackage{numbers, compress}{natbib}

     \usepackage[preprint]{neurips_2020}

\usepackage[utf8]{inputenc} 
\usepackage[T1]{fontenc}    
\usepackage{hyperref}       
\usepackage{url}            
\usepackage{booktabs}       
\usepackage{amsfonts}       
\usepackage{nicefrac}       
\usepackage{microtype}      
\usepackage{xcolor}
\usepackage{graphicx}
\usepackage{subfigure}
\usepackage{wrapfig}
\usepackage{amsmath,amssymb}
\usepackage[flushleft]{threeparttable}
\usepackage{multirow}
\usepackage{tabularx}
\usepackage{lettrine}
\usepackage{algorithm}
\usepackage{algorithmic}
\usepackage{amsthm}
\usepackage{flushend}
\graphicspath{{./figures/}}

\usepackage{makecell}

\newcommand{\xvec}{{\bf{x}}}
\newcommand{\zvec}{{\bf{z}}}

\newcommand{\mline}[2][c]{\begin{tabular}[#1]{@{}c@{}}#2\end{tabular}}

\title{Adversarial Partial Multi-Label Learning} 

\author{%
	Yan Yan$^{1,2}$ \quad \quad Yuhong Guo$^1$\\
	$^1$School of Computer Science, Carleton University, Canada\\
	$^2$School of Computer Science, Northweastern Polytechnical University, China
}

\begin{document}

\maketitle

\begin{abstract}
Partial multi-label learning (PML),
which tackles the problem of learning multi-label prediction models
from instances with overcomplete noisy annotations,
has recently started gaining attention from the research community.
In this paper, we propose a novel adversarial learning 
model, PML-GAN, under a generalized encoder-decoder framework
for partial multi-label learning. 
The PML-GAN model uses a disambiguation network to 
identify irrelevant labels and uses a multi-label prediction network
to map the training instances to their disambiguated label vectors,
while deploying a generative adversarial network 
as an inverse mapping from label vectors to data samples in the input feature space.
The learning of the overall model corresponds to a minimax adversarial game, which
enhances the correspondence of input features with the output labels in a bi-directional mapping.
Extensive experiments 
are conducted on both synthetic and real-world partial multi-label datasets,
while
the proposed model demonstrates the state-of-the-art performance
for partial multi-label learning.
\end{abstract}

\section{Introduction}
\label{introduction}
In partial multi-label learning (PML), 
each training instance is assigned multiple candidate labels which are only partially relevant;
that is, some irrelevant noise labels are assigned together with the ground-truth labels.
As it is typically difficult and costly to precisely annotate instances for 
multi-label data \cite{xie2018partial},
the task of PML naturally arises in many real-world scenarios with crowdsource annotations. 
In such a scenario, in order to collect the complete set of positive labels for each data instance, 
one can gather all labels provided by multiple annotators to
form the candidate label set, 
which is usually overcomplete and contains additional noisy labels beyond all the true labels,
leading to the PML problem. 
Figure \ref{crowd} presents such an example of overcompletely 
annotated training image for object recognition, 
where the 
candidate labels provided by crowdsource annotators 
cover all the ground truth labels (in black color)
and some irrelevant noise labels (in red color). 
PML is much more challenging than standard multi-label learning as the true labels are hidden among
irrelevant labels and the number of true labels is unknown.
The goal of PML is to learn a good multi-label prediction model from such a partial label training set,
and hence reduce the annotation cost.

An intuitive strategy of PML 
is to treat all candidate labels as relevant ground truth, 
thus any off-the-shelf multi-label classification method
can be adapted to induce an expected multi-label predictor
\cite{zhang2014review}. 
This strategy, though simple, cannot work well since taking the noise labels as part of the true labels 
will mislead the multi-label training and induce inferior prediction models.
The PML work in \cite{xie2018partial} 
assumes that each candidate label has a confidence score of being a true label, 
and learns the confidence scores and the classifier in an alternative manner
by minimizing a confidence weighted ranking loss. 
Although this work yields some reasonable results, 
the estimation of label confidence scores is error-prone, especially when noise labels dominate, 
which can seriously impair the classifier's performance. 
The recent work in \cite{xiepartialAAAI20} proposes to perform ground-truth label recovery
and noise label identification simultaneously by exploring the label correlations 
and the relationships between the noise labels and feature representations.
Another recent work in \cite{Fang2018PartialML} 
presents a two-stage PML method.
It estimates the confidence values of the candidate labels using iterative label propagation 
and then chooses the highly confident candidate labels as credible labels 
to induce a multi-label prediction model. 
This work however suffers from the cumulative errors induced in propagation,
which can impact the label confidence estimation and consequently impair the prediction.

\begin{figure*}[t]
\centering
\begin{minipage}[t]{0.32\textwidth}
\includegraphics[width=5.1cm,height=3.35cm]{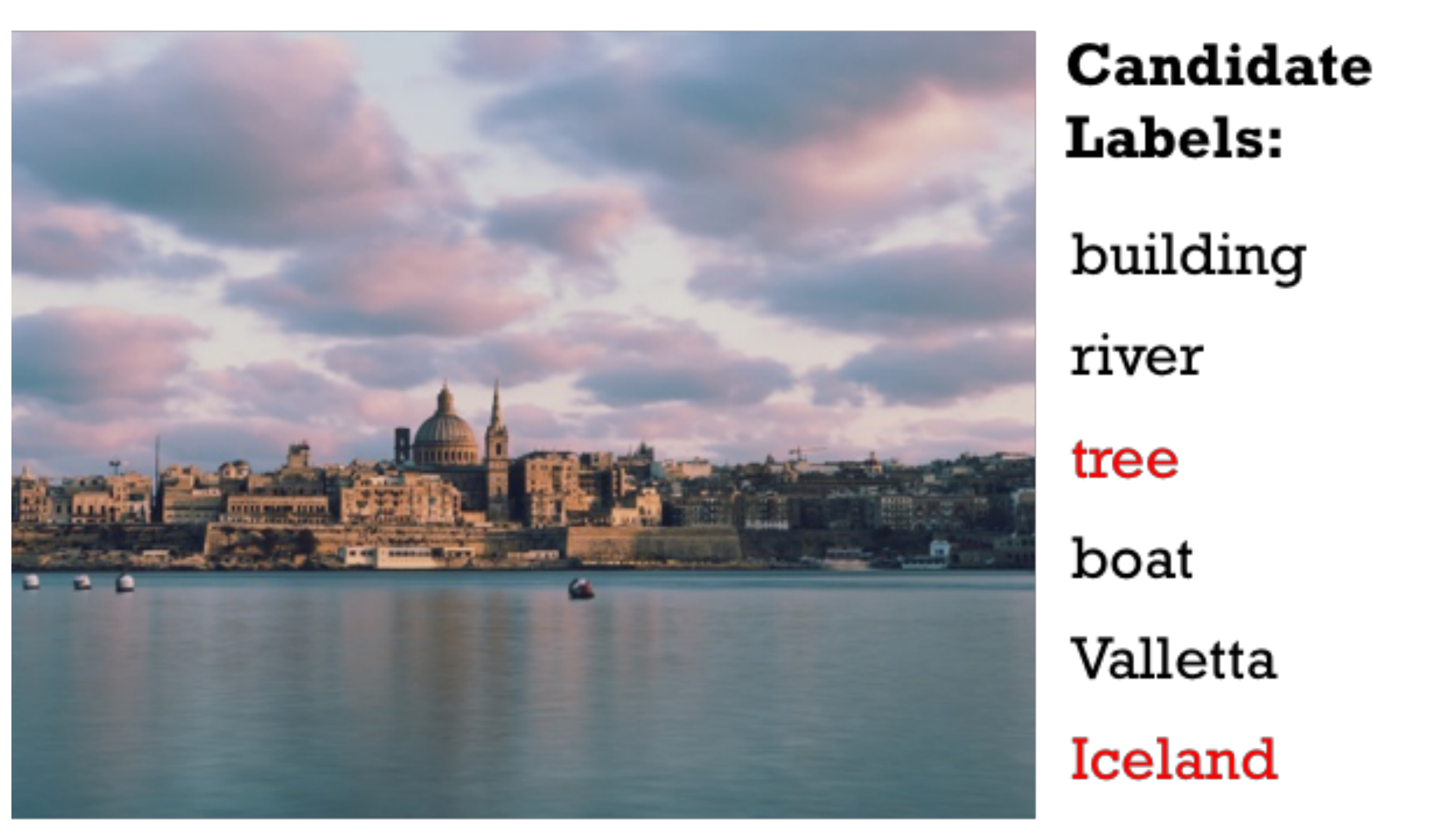}
\vskip -.01in
\caption{An annotated image under the partial multi-label learning (PML) setting.}
\label{crowd}
\end{minipage}
\qquad\;\;
\begin{minipage}[t]{0.60\textwidth}
\centering
\includegraphics[width=7.8cm,height=3.2cm]{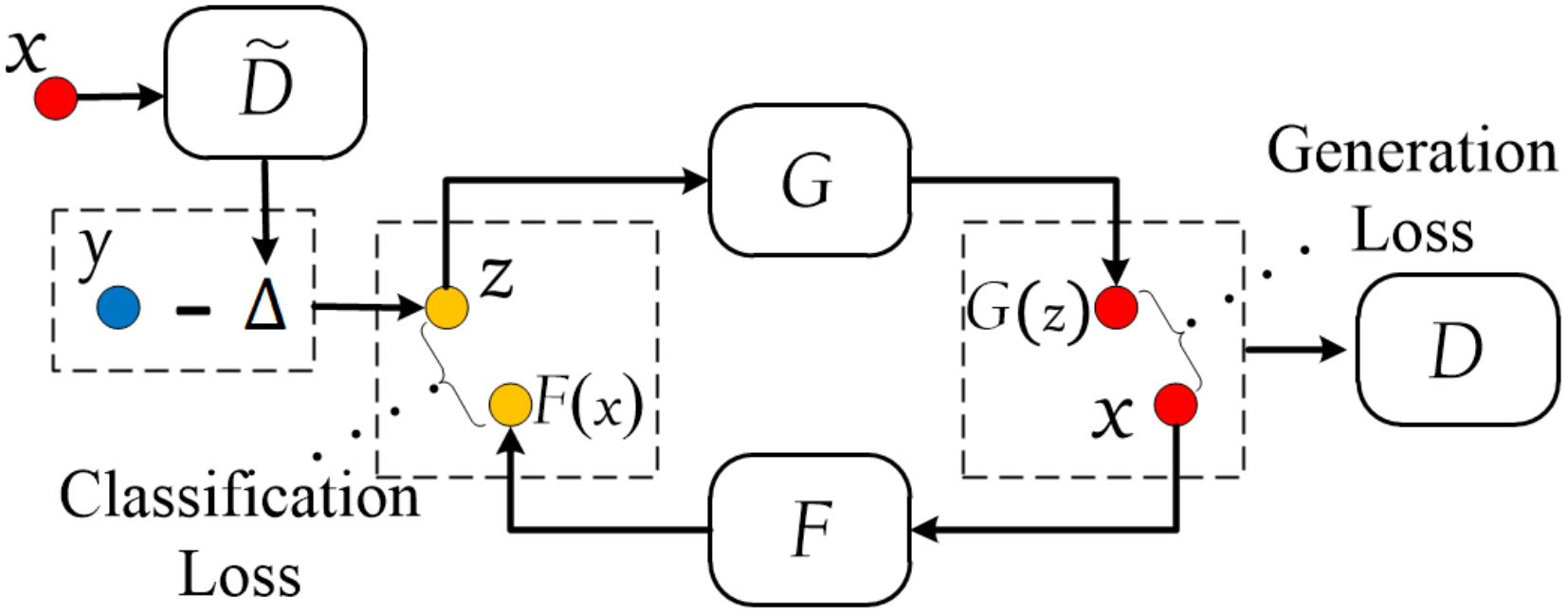}
\caption{The proposed PML-GAN model. 
	It has four component networks: 
generator $G$, disambiguator $\widetilde{D}$, 
predictor $F$, and discriminator $D$.
}
\label{PML-GAN}
\end{minipage}
\vskip -.1in
\end{figure*}

In this paper, we propose a novel adversarial learning model, PML-GAN,
under a generalized encoder-decoder framework
to tackle the partial multi-label learning problem. 
The PML-GAN model comprises four component networks: 
a disambiguation network that predicts the probability of each candidate label being an additive noise for a training instance; 
a prediction network that predicts the disambiguated {\em true} labels of each instance from its input features; 
a generation network that generates samples in the feature space given latent vectors in the label space; 
and a discrimination network that separates the generated samples from the real data. 
The prediction network and disambiguation network together form an encoder that
maps data samples in the input feature space to the disambiguated label vectors,
while the generation network and discrimination network form 
a generative adversarial network (GAN) 
as an inverse decoding mapping 
from vectors in the multi-label space to samples in the input feature space. 
The learning of the overall model corresponds to a minimax adversarial game, 
which enhances the correspondence of input features with the output labels through 
the bi-directional encoder-decoder mapping mechanism,
and consequently boosts multi-label prediction performance.
To the best of our knowledge, this is the first work that exploits a
generative adversarial model based bi-directional mapping mechanism  
for PML.
We conduct extensive experiments on multiple multi-label datasets under partial multi-label learning setting.
The empirical results show the proposed PML-GAN yields the state-of-the-art PML performance.

\section{Related Work}

Multi-label learning is a prevalent classification problem in many real world domains,
where each instance can be assigned into multiple classes simultaneously.
Many multi-label learning methods developed in the literature 
exploit label correlations at different degrees to produce multi-label classifiers \cite{zhang2014review},
including the first order methods \cite{zhang2018binary}, 
second order methods \cite{li2014multi}, and high-order methods 
\cite{burkhardt2018online}.
Nevertheless, standard multi-label learning methods all assume each training instance
is annotated with a complete set of ground truth labels,
which can be impractical in many domains, where the annotations are obtained through crowdsourcing.
With the union of annotations produced by multiple noisy labelers under the crowdsourcing setting,
the partial multi-label learning (PML) problem arises naturally in real world scenarios,
where the set of labels assigned to each training instance 
not only contain the ground truth labels, but also some additional irrelevant labels.

PML is more challenging than standard multi-label learning.  
The previous PML work in \cite{xie2018partial} 
proposes two methods, PML-FP and PML-LC, to estimate the label confidence values and 
optimize the relevance ordering of labels 
by exploring the structural information in both feature and label spaces.
However, due to the inherent property of alternative optimization, 
in these methods, the estimation error of labeling confidence values 
can negatively impact the coupled multi-label predictor. 
The work in \cite{Sun2019PartialML} denoises the observed label matrix based on low-rank and sparse matrix decomposition.
The recent work in \cite{xiepartialAAAI20} proposes to learn the multi-label classifier and noisy label identifier
by exploiting the label correlations as well as exploring the feature-induced noise model 
with the observed noise-corrupted label matrix.
The work in 
\cite{xupartialAAAI20} attempts to recover the label distributions by exploiting the
topological information from the feature space and label correlations from the label space,
and then induces a predictive model by fitting the recovered label distributions.
Another work in \cite{chen2020multi} proposes 
to tackle multi-view PML problem
using graph-based disambiguation.
In another recent work \cite{Fang2018PartialML}, the authors propose to address PML problem using a two-stage strategy.
It first estimates the label confidence value of each candidate label with iterative label propagation,
and then performs multi-label learning over selected credible labels based on the confidence values 
by using pairwise label ranking
(PARTICLE-VLS) or maximum a posteriori reasoning (PARTICLE-MAP).
The work in \cite{wang2019discriminative} also presents a two-stage PML method 
that estimates the label confidence matrix in the first stage.
However, in these two-stage methods, 
the confidence label estimation errors 
can consequently degrade the multi-label learning performance without correction
interaction, especially when there are many noise labels.

Studies on weak learning, partial label learning, and noisy label learning 
have some connections with PML, but address different problems.
Weak label learning tackles the problem of multi-label learning with incomplete labels
\cite{sun2010multi,wei2018learning},
where some ground truth labels are missed out from the annotations.
Partial label learning (PLL) 
tackles multi-class classification under the setting where for each training instance
there is one ground-truth label among the given candidate label set
\cite{cour2011learning,liu2012conditional,zhang2015solving,yu2016maximum,chen2018learning}.
PLL methods cannot be directly applied on the more challenging PML problems, 
as under PML one has unknown numbers of ground truth labels among the candidate label set for each training instance. 
Noisy label learning (NLL) tackles multi-class classification problems 
where some ground-truth labels are {\em replaced} by noise labels
\cite{TheNIPS18,lee2018robust,zhang2018generalized,han2018co,kaneko2019label,lee2019robust}.
The off-the-shelf NLL methods cannot be directly applied on the more challenging PML problems
due to the difference of problem settings.

Generative adversarial networks (GANs) \cite{goodfellow2014generative},
which perform minimax adversarial training over a generation network 
and a discrimination network, 
are one of the most popular generative models since its introduction.
During the past years, 
a vast range of GAN-based adversarial 
learning methods have been developed to address different tasks, including
semi-supervised learning \cite{kumar2017semi,lecouat2018semi}, 
unsupervised learning \cite{jakab2018unsupervised}, 
and learning with noisy labels~\cite{TheNIPS18}.
The proposed work in this paper however is the first one that exploits 
generative adversarial models
for partial multi-label learning.

\section{Proposed Approach}

In this section, we present the proposed adversarial partial multi-label learning model, PML-GAN,
under the following setting.
Assume we have a training set $S  = (X,Y) = \{({\bf x}_i,{\bf y}_i)\}\operatorname{}\limits_{i=1}^N$, 
where ${\bf x}_i\in \mathbb{R}^d $ denotes the input feature vector for the $i$-th instance,
and ${\bf y}_i \in \{0,1\}^L$ is the corresponding annotated label indicator vector.
There are multiple 1 values in each  ${\bf y}_i$,  which indicate either the ground truth labels 
or the additional mis-annotated noise labels.
We aim to learn a good multi-label prediction model 
from this partially labeled training set.

The proposed PML-GAN model is illustrated in Figure~\ref{PML-GAN}, which 
comprises four component networks:  
disambiguation network $\widetilde{D}$, prediction network $F$, 
generation network $G$ and discrimination network $D$.
The four components coordinate with and enhance each other 
under an encoder-decoder learning framework,
which forms inverse mappings
between the instance vectors in the input feature space 
and the continuous label vectors in the output class label space. 
Below we present these model components, the learning objective and the training algorithm in details.

\subsection{Prediction with Disambiguated Labels}

Comparing to standard multi-label learning,
the main difficulty of PML is that the annotated labels $\{{\bf y}_i\}$ in the training data
contain additive noise labels. 
The main challenge lies in identifying the ground truth labels ${\bf z}^*_i$
from each annotated candidate label vectors ${\bf y}_i$;
that is dropping the additional 1s from each candidate label vector ${\bf y}_i$.
We propose to tackle this challenge by using a disambiguation network 
$\widetilde{D}: \Omega_{\bf x}\rightarrow\Omega_\Delta$ 
($\Omega_{\cdot}$ denotes the corresponding domain space),
which predicts the irrelevant labels for a given instance. 
Hence the true label indicator vector ${\bf z}^*_i$ can be recovered as 
${\bf z}^*_i=\mbox{ReLU}({\bf y}_i -\Delta_i)$ in the ideal case,
where $\Delta_i\geq 0$ denotes the output of the disambiguation network $\widetilde{D}({\bf x}_i)$,
and ReLU$(\cdot)=\max(\cdot, 0)$ denotes the commonly used rectified linear unit activation function.
ReLU is used here to ensure the disambiguation effort
is only counted on the candidate labels. 
Then we can learn a prediction network 
$F: \Omega_{\bf x}\rightarrow \Omega_{\bf z}$,
i.e., a multi-label classifier,
to predict the disambiguated ground truth labels for each instance. 

Although the label indicator vectors in the training data are provided as discrete values, 
it is difficult for either the disambiguation network or the prediction network
to directly produce discrete output values.
Instead, by using a sigmoid activation function 
on the last layer of each network,
$\widetilde{D}({\bf x})$ and $F({\bf x})$ can predict the probability
of each class label being the additive irrelevant label and the ground truth label respectively.
With the disambiguation network and prediction network,
we can perform partial multi-label learning by minimizing the 
{\em classification loss} on training data $S$:
\begin{align}
\label{predloss}
\min_{F,\widetilde{D}}&\quad \mathcal L_{c}(X,Y;F,\widetilde{D}) 
	=\operatorname*{\sum}\limits_{({\bf x}_i,{\bf y}_i) \sim S} \ell_c (F({\bf x}_i),{\bf z}_i)
\\
\mbox{s.t.} &\quad
{\bf z}_i=\mbox{ReLU}({\bf y}_i -\Delta_i),
	\; \Delta_i =\widetilde{D}({\bf x}_i),\; \forall ({\bf x}_i,{\bf y}_i) \sim S 
\nonumber	
\end{align}
where ${\bf z}_i$ denotes the disambiguated label confidence vector with continuous values in $[0,1]$,
which can be viewed as a relaxation of a true label indicator vector, 
while $\ell_c(\cdot,\cdot)$ 
denotes the cross-entropy loss between the predicted probability of each label 
and its confidence of being a ground-truth label.
We expect that
the disambiguation network and the prediction network
can coordinate with each other to mutually minimize
this disambiguated classification loss.

\subsection{Inverse Mapping with GANs}

The prediction network can be viewed as an encoder that
maps data samples in the input feature space to the disambiguated label vectors.
To enhance the label disambiguation and hence improve multi-label classification,
we propose to conduct an inverse decoding mapping 
from label vectors ${\bf \hat z}\in [0,1]^L$ in the continuous label vector space 
to samples in the input feature space. 
In particular, we propose to deploy a generative adversarial network (GAN) model
to transform continuous label vectors in the label space 
into samples in the input feature space. 
The GAN model comprises a generation network $G$ and a discrimination network $D$.
Given a label vector $\hat {\bf z}$ sampled from a prior 
distribution $P(\hat {\bf z})$, 
which can be viewed as a low-dimensional representation vector,
one can generate a sample $\hat {\bf x}$ using the generation network, 
${\bf x}=G(\hat {\bf z})$.
A two-class discriminator $D$ is used to discriminate the generated samples from the real samples in $S$.
The training of the GAN model is a minimax optimization problem over an {\em adversarial loss} function:
\begin{align}
    \label{adversarial_loss}
	\min_G\max_D\, \mathcal L_{adv}(G, D, S)=\mathbb{E}_{{\bf x}_i \sim S} [\mathrm{log} D({\bf x}_i)] +
	\mathbb{E}_{\hat {\bf z} \sim P(\hat {\bf z})} [\mathrm{log} (1\!-\!D(G(\hat {\bf z})))],
\end{align}
where the discriminator $D$ tries to maximally distinguish the generated samples $G(\hat {\bf z})$ from the real data samples in $S$,
and the generator $G$ tries to generate samples that are similar to the real data as much as possible such
that the discriminator cannot tell the difference.

In theory, the samples generated by the adversarially trained generator $G$ 
can have an identical distribution with the real data $S$~\cite{goodfellow2014generative}.
To further ensure the generator $G$ can provide an inverse mapping function 
from low-dimensional vectors in the label space to samples in the feature space,
we further propose to decode the disambiguated training label vectors into the training samples $S$ with $G$ 
by deploying a {\em generation loss}:  
\begin{align}
    \label{generation_loss}
	\mathcal L_{g}(G, S) &= \operatorname*{\sum}\limits_{({\bf x}_i, {\bf y}_i) \sim S} \ell_g (G({\bf z}_i),{\bf x}_i),\quad 
	\mbox{with}\;\; {\bf z}_i=\mbox{ReLU}({\bf y}_i -\widetilde{D}({\bf x}_i)),
\end{align}
where $\ell_g(\cdot,\cdot)$ 
measures the generation loss on each training instance, 
which can be a least squares function. 
This generation loss can 
enhance the label disambiguation 
and improve multi-label learning.

\begin{algorithm}[t!]
\caption{Minibatch stochastic gradient descent training of PML-GAN}
\label{alg_1}
\begin{algorithmic}[0] 
\STATE \textbf{Input}: training set $S$;\; 
	trade-off parameter $\beta$;\; 
	$k$-- \# of update steps for the discriminator.\\[.5ex]
\STATE \textbf{for} number of training iterations \textbf{do} \\
   \STATE \hspace{0.3cm} Sample a minibatch of \emph{m} samples 
	$\{({\bf x}_i,{\bf y}_i)\}$ from training set $S$. \\[.2ex]
	\STATE \hspace{0.3cm} Sample \emph{n} label vectors 
	$\{\hat {\bf z}_i\}_{i=1}^n$ from a prior $P(\hat {\bf z})$. \\[.2ex]
   \STATE \hspace{0.3cm} Update the network parameters of $G, \widetilde{D}, F$ by
   descending with their stochastic gradients:
{\small  
\begin{align*}   
\quad\nabla_{\Theta_{G,\widetilde{D},F}}\! 
	\left\{\!\!
	\begin{array}{l}	
	\frac{1}{n}\sum_{i=1}^n \beta [\mathrm{log}(1-D(G(\hat {\bf z}_i)))] \, + \\[.2ex]
	\frac{1}{m}\sum_{i=1}^m\Big[\ell_c \big(F({\bf x}_i),\mbox{ReLU}({\bf y}_i \!-\!\widetilde{D}({\bf x}_i))\big)+
	\ell_g \big(G(\mbox{ReLU}({\bf y}_i -\widetilde{D}({\bf x}_i))),{\bf x}_i\big)\Big]\; 
	\end{array}
	\!\!\!\right\}
\end{align*}   
\vskip -.1in
}
   \STATE \hspace{0.3cm} \textbf{for} r=1:\emph{k} \textbf{do} \\
	\STATE \hspace{0.6cm} Sample \emph{n} label vectors 
	$\{\hat {\bf z}_i\}_{i=1}^n$ from a prior $P(\hat {\bf z})$. \\[.2ex]
        \STATE \hspace{0.6cm} Update the parameters of the discrimination network 
        \STATE \hspace{0.6cm} by ascending with its stochastic gradient:
{\small  
\begin{align*}   
\qquad\nabla_{\Theta_D} 
	\, \beta\,\Big[ \frac{1}{m}\sum_{i=1}^m \mathrm{log} D({\bf x}_i)
	+\frac{1}{n}\sum_{i=1}^n \mathrm{log} (1\!-\!D(G(\hat {\bf z})))\Big]
\end{align*}   
}
	\vspace{ -.1in}	
    \STATE \hspace{0.3cm} \textbf{end for}
\STATE \textbf{end for} \\
\end{algorithmic}
\end{algorithm}

\subsection{Learning with PML-GANs}

By integrating the classification loss in Eq.(\ref{predloss}), 
the adversarial loss in Eq.(\ref{adversarial_loss}), and the generation loss in Eq.(\ref{generation_loss}) together,
we obtain the following minimax optimization problem for the proposed PML-GAN model:
\begin{align}
	\min_{G,\widetilde{D},F}\!\max_D  &\;\; 
	\mathbb{E}_{({\bf x}_i,{\bf y}_i) \sim S} 
	\Big(\ell_c (F({\bf x}_i),{\bf z}_i) +\ell_g (G({\bf z}_i),{\bf x}_i)\Big) +
\nonumber\\	
	&
	\beta\,\Big(\! \mathbb{E}_{{\bf x}_i \sim S} [\mathrm{log} D({\bf x}_i)]
	+\mathbb{E}_{\hat {\bf z} \sim P(\hat {\bf z})} [\mathrm{log} (1\!-\!D(G(\hat {\bf z})))]\!\Big)
\label{eq:objective}	
\\
	\mbox{s.t.} \quad &
{\bf z}_i=\mbox{ReLU}({\bf y}_i -\widetilde{D}({\bf x}_i)),
	\;\; \forall ({\bf x}_i,{\bf y}_i) \sim S 
\nonumber
\end{align}
where $\beta$ is a trade-off hyperparameter that controls 
the relative importance of the adversarial loss;
the objective function can be denoted as $\mathcal{L}(G,\widetilde{D},F,D)$. 
The learning of the overall model corresponds to a minimax adversarial game, 
which enhances 
the bi-directional mapping between the feature and label vector spaces, 
and consequently boosts multi-label prediction performance.

We perform training using a minibatch based stochastic gradient descent algorithm. 
In each iteration of the training, the minimization over $G, \widetilde{D}, F$
and the maximization over $D$ are conducted alternatively.
The overall training algorithm is presented in Algorithm \ref{alg_1}.

\section{Theoretical Results}
\begin{wrapfigure}{r}{0.45\textwidth} 
\centering
\vskip -.15in
\includegraphics[width=1.50in,height=0.70in]{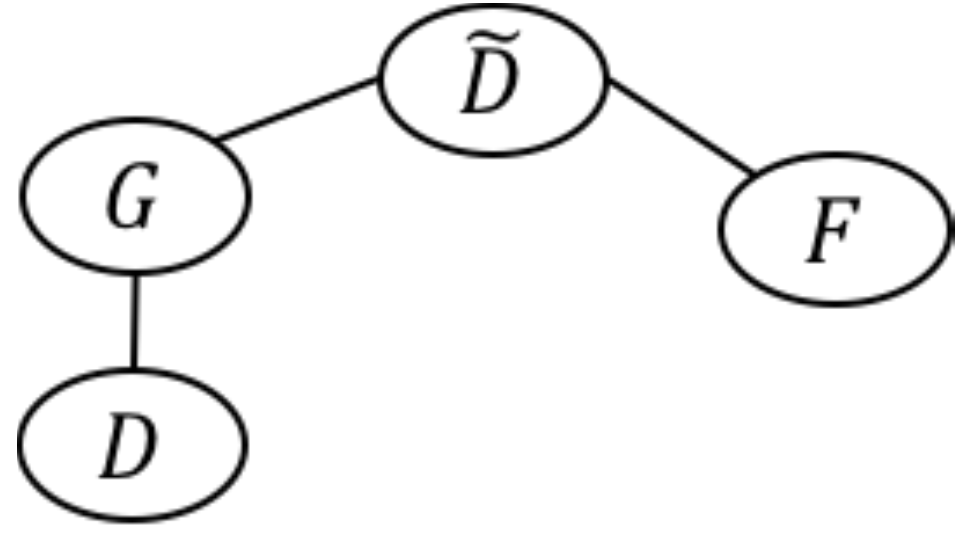}
\caption{Dependence graph of PML-GAN.}
\label{fig:DG}
\end{wrapfigure} 

In the proposed PML-GAN model, given the generator $G$, the discriminator $D$ is conditionally
independent from the predictor $F$ and the disambiguator $\widetilde{D}$. 
Between $G, F$ and $\widetilde{D}$, $G$ and 
$F$ are conditionally independent from each other given $\widetilde{D}$. 
Their independence relationship can be illustrated using the undirected dependence graph in Figure~\ref{fig:DG}.
Based on these conditional independence relationships, we have the following optimality results.

\newtheorem{prop}{Proposition}
\begin{prop}
\label{prop1}
For any $G$, $\widetilde{D}$, and $F$, the optimal discriminator $D$ is given by
\begin{align}
			D^*_{G, \widetilde{D}, F}({\bf x}) 
			=D^*_{G}({\bf x}) 
			= {p_S({\bf x})}/{\big(p_S({\bf x})+p_g({\bf x})\big)}
\end{align}
where $p_S(\cdot)$ and $p_g(\cdot)$ denote the probability distributions of real and generated data respectively. 
\end{prop}
\begin{proof}
\vskip -.1in
Due to the conditional independence relationship between $D$ and $\{F, \widetilde{D}\}$, 
the optimal discriminator $D$ only depends on the generator $G$.
Given fixed $G$, the optimal discriminator can be derived in the same way as in the standard GAN 
\cite[Proposition 1]{goodfellow2014generative}.
\vskip -.2in
\end{proof}

\begin{prop}
Assume the model has sufficient capacity. 
Let $C(G,\widetilde{D},F)=\max_D \mathcal{L}(G,\widetilde{D},F,D)$. 
$H(\cdot)$ denotes an entropy function. 
Given fixed $\widetilde{D}$, the minimum of $C(G,\widetilde{D},F)$
is lower bounded by 
$ \mathbb{E}_{({\bf x}_i,{\bf y}_i) \sim S}\; H\big(\mbox{ReLU}({\bf y}_i -\widetilde{D}({\bf x}_i))\big) -\beta\,\log 4$, 
which can be achieved when $F({\bf x}_i)=\mbox{ReLU}({\bf y}_i -\widetilde{D}({\bf x}_i))$, 
$G(F({\bf x}_i))={\bf x}_i$, and $p_g = p_S$. 
\end{prop}

\begin{proof}
This proposition suggests that $F$ and $G$ should be inverse mapping functions for each other in the ideal optimal case.
Based on the solution for optimal discriminator $D^*$ in Proposition 1, we have:
{\small
\begin{align*}
	&\;\mathcal L_{adv}(G, D^*, S)\\
	= &\;
\mathbb{E}_{{\bf x} \sim p_S} [\mathrm{log} D^*_G({\bf x})]
	+\mathbb{E}_{\hat {\bf z} \sim P(\hat {\bf z})} [\mathrm{log} (1\!-\!D^*_G(G(\hat {\bf z})))]
\\	
	= &\;
\mathbb{E}_{{\bf x} \sim p_S} [\mathrm{log} D^*_G({\bf x})]
	+\mathbb{E}_{{\bf x} \sim p_g} [\mathrm{log} (1-D^*_G({\bf x}))]
\\	
	= &\;
\mathbb{E}_{{\bf x} \sim p_S} \left[\mathrm{log} \frac{p_S(\xvec)}{p_S(\xvec)+p_g(\xvec)}\right]
+\mathbb{E}_{{\bf x} \sim p_g} \left[\mathrm{log}\frac{p_g(\xvec)}{p_S(\xvec)+p_g(\xvec)}\right]
\end{align*}
}
Hence,
{\small
\begin{align*}
	&\; C(G,\widetilde{D},F)
	=\max_D \mathcal{L}(G,\widetilde{D},F,D)\\ 
	&=\left\{
	\begin{array}{l}
\!\!\!	\mathbb{E}_{({\bf x}_i,{\bf y}_i) \sim S}\,
		\ell_c (F({\bf x}_i),\mbox{ReLU}({\bf y}_i -\widetilde{D}({\bf x}_i)))\,+ \\[.5ex]
\!\!\!	\mathbb{E}_{({\bf x}_i,{\bf y}_i) \sim S}\,
		\ell_g (G(\mbox{ReLU}({\bf y}_i -\widetilde{D}({\bf x}_i))),{\bf x}_i)\, +\\[.5ex]
		\!\!\!	\beta\,\Big( \mathbb{E}_{{\bf x} \sim p_S} [\mathrm{log} \frac{p_S(\xvec)}{p_S(\xvec)+p_g(\xvec)}]
		+\mathbb{E}_{{\bf x} \sim p_g} [\mathrm{log}\frac{p_g(\xvec)}{p_S(\xvec)+p_g(\xvec)}]\Big)
	\end{array}
\!\!	\right\} 
\end{align*}
}
Note given fixed $\widetilde{D}$, $F$ is conditionally independent from $G$ and $D$.
Hence the minimization of $C(G,\widetilde{D},F)$ over $F$ 
can be independently conducted from the minimization over $G$. 
Let $\zvec_i = \mbox{ReLU}({\bf y}_i -\widetilde{D}({\bf x}_i))$. 
With the cross-entropy loss function $\ell_c(\cdot,\cdot)$, we have:
{\small
\begin{align}
	&\;\;\quad\min_F \; C(G,\widetilde{D},F)
	\nonumber
\\
	& \; \equiv\;
	\min_F\; \mathbb{E}_{({\bf x}_i,{\bf y}_i) \sim S} 
	\quad \ell_c \left(F({\bf x}_i),\mbox{ReLU}({\bf y}_i -\widetilde{D}({\bf x}_i))\right) 
	\nonumber
	\\
	& \; \equiv\;
	\min_F\; \mathbb{E}_{({\bf x}_i,{\bf y}_i) \sim S} 
	\left[- \zvec_i^\top\!\log F(\xvec_i)
	\!-\! (1\!-\!\zvec_i)^\top\!\log (1\!-\!F(\xvec_i))\right]
	\nonumber
	\\
	& \; \equiv\;
	\min_F\; \mathbb{E}_{({\bf x}_i,{\bf y}_i) \sim S} 
	\quad H(\zvec_i) + \mbox{KL}(\zvec_i\parallel F(\xvec_i))
	\nonumber
	\\
	& \;\geq\; \mathbb{E}_{({\bf x}_i,{\bf y}_i) \sim S}\; H(\zvec_i) 
	\label{minF}
\end{align}
}
where $H(\cdot)$ denotes the entropy over a binomial distribution vector 
and KL$(\cdot)$ denotes the KL-divergence between two sets of binomial distributions.
Assume sufficient capacity for $F$, 
in the ideal case the minimum can be reached when 
the predictor obtains the same distributions as the $\zvec_i$; that is 
\begin{align}
	& F^*({\bf x}_i) = \zvec_i = \mbox{ReLU}({\bf y}_i -\widetilde{D}({\bf x}_i)),
	\;\,\forall ({\bf x}_i,{\bf y}_i) \in S
	\label{minFcond}
\end{align}

Next let's consider the minimization problem over $G$. 
Note $G$ is involved in both the generation loss and adversarial loss.
If we could find solutions that lead to minimals in both losses separately, 
we can guarantee a minimal in the united loss. 
Based on \cite[Theorem 1]{goodfellow2014generative}, 
the adversarial loss part
in  $C(G,\widetilde{D},F)$ can be rewritten as
{\small
\begin{align}
	&\quad\beta \mathcal{L}_{adv} 
	\nonumber\\
	&=\beta\,\Big( \mathbb{E}_{{\bf x} \sim p_S} [\mathrm{log} \frac{p_S(\xvec)}{p_S(\xvec)+p_g(\xvec)}]
		+\mathbb{E}_{{\bf x} \sim p_g} [\mathrm{log}\frac{p_g(\xvec)}{p_S(\xvec)+p_g(\xvec)}]\Big)
	\nonumber\\
	   &= \beta\,\Big(\mbox{KL}(p_S,\frac{p_S+p_g}{2}) - \log 2+ \mbox{KL}(p_g,\frac{p_S+p_g}{2}) -\log 2\Big)
	\nonumber\\
	   &= \beta\,\Big(\mbox{KL}(p_S,\frac{p_S+p_g}{2}) + \mbox{KL}(p_g,\frac{p_S+p_g}{2}) -\log 4\Big)
	\nonumber\\
	   &\geq -\beta\,\log 4
	   \label{minGadv}
\end{align}
}
where the minimal can be achieved when $p_S=p_g$ which leads to zero KL-divergence values.
The generation loss part (with least squares loss function)
in  $C(G,\widetilde{D},F)$ can be rewritten as
\begin{align}
&\quad\mathbb{E}_{({\bf x}_i,{\bf y}_i) \sim S} \Big(
\ell_g (G(\mbox{ReLU}({\bf y}_i -\widetilde{D}({\bf x}_i))),{\bf x}_i)\Big)
	\nonumber\\
=&\quad\mathbb{E}_{({\bf x}_i,{\bf y}_i) \sim S} 
\;  \|G(\mbox{ReLU}({\bf y}_i -\widetilde{D}({\bf x}_i)))-{\bf x}_i\|^2
	\nonumber\\
\geq & \quad 0
\label{minGg}
\end{align}
where the minimal 0 can only be achieved when 
\begin{align}
G(\mbox{ReLU}({\bf y}_i -\widetilde{D}({\bf x}_i)))={\bf x}_i,
	\quad\forall ({\bf x}_i,{\bf y}_i) \in S
\label{minGgcond}
\end{align}
The optimal condition above can be satisfied simultaneously together with
the condition $p_g=p_S$. Together with the condition in (\ref{minFcond}),
these conditions lead to a lower bound of 
$C(G,\widetilde{D},F)$
and the proposition is proved.
\end{proof}

\section{Experiments}

\subsection{Experimental Setting}
{\bf Datasets.}\quad We conducted experiments on twelve multi-label classification datasets.
Three of them have existing partial multi-label learning settings 
(mirflickr, music\_style and music\_emotion \cite{Fang2018PartialML}).
For each of the other nine datasets \cite{zhang2014review}, 
we transformed it into a PML dataset
by randomly adding irrelevant labels into the candidate label set of each training instance. 
By adding different numbers of irrelevant labels, for each dataset we can create multiple PML variants 
with different average numbers of candidate labels.
Following the setting of 
\cite{xie2018partial}, 
we also 
filtered out the rare labels and kept at most 15 classes in each dataset.
The detailed characteristics of the processed datasets are given in Table~\ref{pmldataset}.

{\bf Comparison Methods.}\quad
We compared our proposed method with five state-of-the-art PML methods and one baseline multi-label learning method.
We adopted a simple but effective neural network based multi-label learning method, 
ML-RBF \cite{zhang2009m}, 
as a baseline method, 
which performs PML by treating all the candidate labels as ground-truth labels.
Then we used five recently developed PML methods for comparison,
including the PML-LC and PML-FP methods from \cite{xie2018partial}, 
the PARTICLE-VLS and PARTICLE-MAP methods from \cite{Fang2018PartialML},
and PML-NI from~\cite{xiepartialAAAI20}.

{\bf Implementation.}\quad
The proposed PML-GAN model has four component networks, all of which are designed
as multilayer perceptrons with Leaky ReLu activation for the middle layers. 
The disambiguator, predictor, and discriminator are all three-layer networks
with sigmoid activation in the output layer, 
while the generator is a five layer network with Tanh activation in the output layer.
Batch normalization is also deployed in the middle three layers of the generation network.
We used the Adam~\cite{kingma2014adam} optimizer in our implementation.
The mini-batch size, $m$, is set to 64.
The hyperparameters \emph{k} (the number of steps for discriminator update) and \emph{n} 
(the number of label vectors sampled) in Algorithm \ref{alg_1} are set to 1 and $2^{10}$ respectively.
The hyperparameter $\beta$ is chosen from $\{0.001, 0.01, 0.1, 1, 10\}$ 
based on the classification loss value $\mathcal{L}_c$ in the training objective function;
that is, the $\beta$ value that leads to the smallest training $\mathcal{L}_c$ loss will be chosen.

\begin{table*}[t!]
	\centering
\vskip -.12in	
	\caption{Information of the experimental data sets. The number of instances, features and classes are recorded.
The ``\textit{avg.\#CLs}'' column lists the average number of candidate labels in each PML set.
	}
	\vskip .1in
	\label{pmldataset}
	\setlength{\tabcolsep}{2pt} 
	\resizebox{1\textwidth}{!}{%
		\begin{tabular}{lcccc|lccccc}
			\hline
			\mline{{Dataset}} &{\#Inst.}       &{\#Feats}      &{\#Classes}      &{avg.\#CLs}     
			&{Dataset}        &{\#Inst.}       &{\#Feats}      &{\#Classes}      &{avg.\#CLs}     
			\\
            \hline
            music\_emotion      &6833                     &98                      &11                          &5.29
&music\_style       &6839                     &98                      &10                          &6.04
\\
            mirflickr           &10433                    &100                     &7                           &3.35
&image              &2000                     &294                     &5                           &2,3,4
\\
            scene               &2407                     &294                     &6                           &3,4,5
&yeast              &2417                     &103                     &14                          &9,10,11,12
\\
            enron  	            &1702                     &1001                    &15                          &8,9,10,11,12,13
&corel5k  	        &5000                     &499                     &15                          &8,9,10,11,12,13
\\
            eurlex\_dc  	    &8636                     &100                     &15                          &8,9,10,11,12,13
&eurlex\_sm  	    &12679                    &100                     &15                          &8,9,10,11,12,13
\\
            delicious  	        &14000                    &500                     &15                          &8,9,10,11,12,13
&tmc2007  	        &28596                    &49060                   &15                          &8,9,10,11,12,13
\\\hline
		\end{tabular}%
	}
\vskip -.1in	
\end{table*}

\begin{table*}[t!]
	\centering
\vskip -.12in	
	\caption{ Comparison results of in terms of Hamming loss, ranking loss, one error and average precision.
	The best results are presented in bold font. The average number of candidate labels is 
	presented under the column ``avg.\#C.Ls''. }
	\vskip .1in
	\label{pmlresult}
	\setlength{\tabcolsep}{4pt} 
	\resizebox{\textwidth}{!}{%
		\begin{tabular}{l|c|ccccccc}
			\hline
			\mline{Data set} &avg.\#C.Ls &\textsc{Pml-Gan} &\textsc{Pml-Ni} &\textsc{\makecell{Particle\\-Vls}} &\textsc{\makecell{Particle\\-Map}} &\textsc{Pml-lc} &\textsc{Pml-fp}  &\textsc{Ml-rbf} \\
			\hline
            \multicolumn{9}{l}{Hamming loss (the smaller, the better)} \\  
            \hline
            music\_emotion  &5.29  	     &\textbf{.200$\pm$.004} &.212$\pm$.003    &.212$\pm$.004	         &.215$\pm$.004         	 &.236$\pm$.003           	 	
&.245$\pm$.004	         	           &.779$\pm$.004\\
            music\_style   &6.04  	     &\textbf{.115$\pm$.002} &.116$\pm$.004    &.121$\pm$.003	         &.175$\pm$.005         	 &.126$\pm$.004           	 	
&.126$\pm$.004	         	           &.856$\pm$.001\\
            mirflickr      &3.35  	     &.170$\pm$.003 &\textbf{.167$\pm$.003}    &.178$\pm$.035	         &.189$\pm$.081         	 &.202$\pm$.057           	 	
&.202$\pm$.057	         	           &.748$\pm$.002\\ \hline
            image          &3     	     &\textbf{.202$\pm$.006} &.210$\pm$.009    &.234$\pm$.065	         &.269$\pm$.096         	 &.264$\pm$.072           	 	
&.267$\pm$.063	         	           &.754$\pm$.003\\
            scene          &4  	         &\textbf{.132$\pm$.007} &.175$\pm$.003    &.184$\pm$.037	         &.174$\pm$.035         	 &.178$\pm$.029           	 	
&.187$\pm$.038	         	           &.820$\pm$.001\\ \hline
            yeast&\multirow{6}{*}{10}   &\textbf{.213$\pm$.008}  &.232$\pm$.004    &.226$\pm$.004	         &.220$\pm$.008         	 &.226$\pm$.008           	 	
&.219$\pm$.009	         	           &.694$\pm$.003\\
            enron&  	                 &\textbf{.186$\pm$.003} &.235$\pm$.005    &.197$\pm$.032	         &.190$\pm$.036
             &.206$\pm$.027           	 	
&.206$\pm$.027	         	           &.813$\pm$.004\\
            corel5k&  	                 &\textbf{.118$\pm$.001} &.135$\pm$.003    &.189$\pm$.012	         &.269$\pm$.027         	 &.151$\pm$.008           	 	
&.152$\pm$.008	         	           &.886$\pm$.001\\
            eurlex\_dc&  	             &\textbf{.044$\pm$.001} &.067$\pm$.001    &.061$\pm$.001	         &.064$\pm$.004         	 &.096$\pm$.001           	 	
&.071$\pm$.001	         	           &.933$\pm$.001\\
            eurlex\_sm&  	             &.083$\pm$.002          &.091$\pm$.008    &\textbf{.067$\pm$.001}	 &.076$\pm$.002         	 &.119$\pm$.006           	 	
&.122$\pm$.002	         	           &.885$\pm$.001\\
            delicious&  	             &\textbf{.249$\pm$.002} &.260$\pm$.002    &.260$\pm$.003	         &.290$\pm$.005         	 &.290$\pm$.004           	 	
&.290$\pm$.004	         	           &.712$\pm$.002\\
            tmc2007&  	                 &\textbf{.084$\pm$.001} &.089$\pm$.001    &.090$\pm$.003	         &.110$\pm$.003         	 &.103$\pm$.002           	 	
&.103$\pm$.002	         	           &.857$\pm$.001\\
			\hline	

            \multicolumn{9}{l}{Ranking loss (the smaller, the better)} \\
            \hline
            music\_emotion  &5.29  	     &.242$\pm$.007          &.251$\pm$.007          &.263$\pm$.008	           &\textbf{.240$\pm$.007}         	 &.267$\pm$.009           	 	
&.275$\pm$.010	         	           &.365$\pm$.010\\
            music\_style   &6.04  	     &.145$\pm$.006 &\textbf{.140$\pm$.009}          &.163$\pm$.007	         &.147$\pm$.005         	 &.215$\pm$.005           	 	
&.150$\pm$.005	         	           &.242$\pm$.006\\
            mirflickr      &3.35  	     &\textbf{.124$\pm$.014} &\textbf{.124$\pm$.004} &.227$\pm$.029	         &.129$\pm$.108         	 &.160$\pm$.029           	 	
&.143$\pm$.028	         	           &.195$\pm$.015\\ \hline
            image          &3     	     &\textbf{.191$\pm$.010} &.217$\pm$.008          &.239$\pm$.077	         &.250$\pm$.085         	 &.291$\pm$.134           	 	
&.217$\pm$.120	         	           &.251$\pm$.019\\
            scene        &4  	         &\textbf{.123$\pm$.009} &.213$\pm$.010          &.177$\pm$.049	         &.167$\pm$.060         	 &.192$\pm$.032           	 	
&.238$\pm$.056	         	           &.188$\pm$.014\\ \hline
            yeast&\multirow{8}{*}{10}   &\textbf{.194$\pm$.011}  &.222$\pm$.005          &.203$\pm$.007	         &.208$\pm$.012               &.219$\pm$.011           	 	
&.203$\pm$.008	         	           &.270$\pm$.007\\
            enron&  	                 &\textbf{.182$\pm$.012} &.236$\pm$.013          &.240$\pm$.078	         &\textbf{.182$\pm$.029}      &.239$\pm$.048           	 	
&.239$\pm$.047	         	           &.244$\pm$.010\\
            corel5k&  	                 &\textbf{.295$\pm$.011} &.392$\pm$.009          &.367$\pm$.032	         &.311$\pm$.008         	 &.366$\pm$.035           	 	
&.398$\pm$.025	         	           &.404$\pm$.082\\
            eurlex\_dc&  	             &\textbf{.067$\pm$.005} &.126$\pm$.010          &.150$\pm$.004	         &.085$\pm$.004         	 &.137$\pm$.008           	 	
&.131$\pm$.001	         	           &.135$\pm$.003\\
            eurlex\_sm&  	             &\textbf{.122$\pm$.007} &.246$\pm$.037          &.129$\pm$.007	         &.127$\pm$.009         	 &.282$\pm$.007           	 	
&.182$\pm$.008	         	           &.183$\pm$.003\\
            delicious&  	             &\textbf{.258$\pm$.004} &.287$\pm$.002          &.314$\pm$.005	         &.276$\pm$.004         	 &.277$\pm$.005           	 	
&.276$\pm$.005	         	           &.316$\pm$.003\\
            tmc2007&  	                 &\textbf{.070$\pm$.001} &.077$\pm$.001          &.096$\pm$.008	         &.095$\pm$.007         	 &.082$\pm$.005           	 	
&.080$\pm$.005	         	           &.153$\pm$.002\\
			\hline

            \multicolumn{8}{l}{One error  (the smaller, the better)} \\
            \hline
            music\_emotion  &5.29  	     &\textbf{.450$\pm$.028} &.500$\pm$.014          &.473$\pm$.016	         &.475$\pm$.018         	 &.556$\pm$.028           	 	
&.540$\pm$.027	         	         &.587$\pm$.019\\
            music\_style   &6.04  	     &\textbf{.347$\pm$.016} &.355$\pm$.016          &.374$\pm$.005	         &.399$\pm$.019         	 &.409$\pm$.013           	 	
&.408$\pm$.013	         	         &.385$\pm$.006\\
            mirflickr      &3.35  	     &.236$\pm$.059	         &.307$\pm$.020          &\textbf{.165$\pm$.150} &.229$\pm$.306         	 &.300$\pm$.129           	 	
&.298$\pm$.121	         	         &.338$\pm$.002\\ \hline
            image          &3     	     &\textbf{.342$\pm$.014} &.401$\pm$.028          &.369$\pm$.134	         &.387$\pm$.147         	 &.542$\pm$.191           	 	
&.549$\pm$.174	         	         &.398$\pm$.034\\
            scene          &4  	         &\textbf{.321$\pm$.022} &.413$\pm$.018          &.340$\pm$.078	         &.349$\pm$.082         	 &.497$\pm$.089           	 	
&.523$\pm$.118	         	         &.428$\pm$.022\\    \hline
            yeast&\multirow{8}{*}{10}    &\textbf{.245$\pm$.017} &.290$\pm$.009          &.248$\pm$.019	         &.252$\pm$.018              &.257$\pm$.017           	 	
&.263$\pm$.027	         	         &.408$\pm$.023\\
            enron&  	                 &\textbf{.307$\pm$.035} &.498$\pm$.024          &.411$\pm$.101	         &.351$\pm$.040         	 &.494$\pm$.039           	 	
&.498$\pm$.038	         	         &.495$\pm$.019\\
            corel5k&  	                 &\textbf{.685$\pm$.015} &.792$\pm$.016          &.835$\pm$.025	         &.721$\pm$.035         	 &.784$\pm$.029           	 	
&.787$\pm$.024	         	         &.809$\pm$.015\\
            eurlex\_dc&  	             &\textbf{.307$\pm$.013} &.521$\pm$.015          &.390$\pm$.016	         &.374$\pm$.014         	 &.707$\pm$.014           	 	
&.518$\pm$.011	         	         &.342$\pm$.008\\
            eurlex\_sm&  	             &\textbf{.339$\pm$.013} &.516$\pm$.019          &.350$\pm$.014	         &.360$\pm$.015         	 &.506$\pm$.031           	 	
&.542$\pm$.018	         	         &.340$\pm$.005\\
            delicious&  	             &.368$\pm$.009	         &.415$\pm$.007          &\textbf{.366$\pm$.015} &.414$\pm$.018         	 &.401$\pm$.015           	 	
&.399$\pm$.013	         	         &.450$\pm$.009\\
            tmc2007&  	                 &.202$\pm$.007	         &.214$\pm$.008          &\textbf{.194$\pm$.029} &.267$\pm$.018         	 &.235$\pm$.019           	 	
&.236$\pm$.019	         	         &.388$\pm$.006\\    \hline

            \multicolumn{9}{l}{Average precision  (the larger, the better)} \\
            \hline
            music\_emotion  &5.29  	     &\textbf{.621$\pm$.013} &.598$\pm$.007          &.605$\pm$.012	         &.612$\pm$.009         	 &.574$\pm$.013           	 	
&.568$\pm$.014	         	           &.506$\pm$.012\\
            music\_style   &6.04  	     &\textbf{.732$\pm$.010} &.729$\pm$.012          &.715$\pm$.009	         &.709$\pm$.009         	 &.702$\pm$.008           	 	
&.703$\pm$.008	         	           &.646$\pm$.010\\
            mirflickr      &3.35  	     &.777$\pm$.027          &.787$\pm$.008          &.678$\pm$.027	         &\textbf{.791$\pm$.202}         	 &.736$\pm$.043           	 	
&.758$\pm$.039	         	           &.676$\pm$.048\\ \hline
            image          &3     	     &\textbf{.775$\pm$.010} &.740$\pm$.013          &.741$\pm$.090	         &.729$\pm$.086         	 &.644$\pm$.131           	 	
&.725$\pm$.119	         	           &.723$\pm$.021\\
            scene        &4  	         &\textbf{.801$\pm$.012} &.688$\pm$.011          &.750$\pm$.074	         &.753$\pm$.064         	 &.689$\pm$.047           	 	
&.710$\pm$.079	         	           &.728$\pm$.015\\ \hline
            yeast&\multirow{8}{*}{10}    &\textbf{.732$\pm$.014} &.701$\pm$.004          &.724$\pm$.010	         &.714$\pm$.010         	 &.721$\pm$.012           	 	
&.728$\pm$.010	         	           &.634$\pm$.008\\
            enron&  	                 &\textbf{.665$\pm$.019} &.580$\pm$.009          &.595$\pm$.099	         &.661$\pm$.047         	 &.556$\pm$.041           	 	
&.575$\pm$.041	         	           &.560$\pm$.009\\
            corel5k&  	                 &\textbf{.441$\pm$.012} &.345$\pm$.010          &.377$\pm$.025	         &.415$\pm$.008         	 &.345$\pm$.027           	 	
&.384$\pm$.021	         	           &.334$\pm$.008\\
            eurlex\_dc&  	             &\textbf{.797$\pm$.009} &.704$\pm$.022          &.692$\pm$.013	         &.751$\pm$.008         	 &.693$\pm$.019           	 	
&.716$\pm$.014	         	           &.710$\pm$.000\\
            eurlex\_sm&  	             &\textbf{.720$\pm$.009} &.558$\pm$.023          &.705$\pm$.009	         &.683$\pm$.011         	 &.438$\pm$.016           	 	
&.679$\pm$.011	         	           &.656$\pm$.000\\
            delicious&  	             &\textbf{.630$\pm$.006} &.597$\pm$.003          &.596$\pm$.007	         &.601$\pm$.008         	 &.607$\pm$.007           	 	
&.608$\pm$.006	         	           &.576$\pm$.004\\
            tmc2007&  	                 &\textbf{.821$\pm$.002} &.807$\pm$.003          &.799$\pm$.013	         &.759$\pm$.013         	 &.793$\pm$.012           	 	
&.794$\pm$.012	         	           &.662$\pm$.003\\
			\hline	
		\end{tabular}%
	}
\vskip -.2in
\end{table*}

\subsection{Comparison Results}

We compared the proposed PML-GAN method with the six comparison methods on the twelve datasets.
For each dataset, we randomly select 80\% of the data for training and use the remaining 20\% for testing.
We repeat each experiment 10 times with different random partitions of the datasets.
The comparison test results in terms of four commonly used evaluation metrics 
(Hamming loss, ranking loss, one error and average precision) 
\cite{zhang2014review} are reported 
in Table \ref{pmlresult}. 
The results are the means and standard deviations over the 10 repeated runs.
We can see that the methods specially developed for PML problems all outperform 
the baseline multi-label neural network classifier, ML-RBF, in most cases. 
But it is difficult to beat the baseline competitor on all the datasets
with different evaluation metrics.
Among the total 48 cases over 12 datasets and 4 evaluation metrics, 
PML-NI, PARTICLE-VLS, PARTICLE-MAP, PML-LC and PML-FP outperform ML-RBF in 
39, 42, 45, 35 and 40 cases respectively.
By contrast, the proposed PML-GAN method outperforms ML-RBF {\em consistently} across all the 48 cases
with remarkable performance gains.
Even comparing with all the other five PML methods, PML-GAN produced
the best results in 40 out of the total 48 cases. 
Moreover, the performance gains yield by PML-GAN over all the other methods 
are quite notable in many cases.
For example, in terms of average precision, PML-GAN outperforms the 
best alternative comparison method by 
4.6\%, 4.8\%, and 3.4\%
on {\em eurlex\_dc, scene} and {\em image} respectively.
These results clearly demonstrate the effectiveness of the proposed PML-GAN model. 

The results reported in Table~\ref{pmlresult} and discussed above
are produced on each dataset with a selected average number of candidate labels.
As shown in Table~\ref{pmldataset}, we have multiple PML variants with different numbers of candidate labels
for nine of the datasets in the list. 
In total this provides us 49 PML datasets.
We hence also conducted experiments on each of these 49 variant datasets,
by comparing the proposed PML-GAN with each of the other methods 
in terms of the 4 evaluation metrics.
In total there are 196 comparison cases for each pair of methods.
For the comparison of ``PML-GAN vs other method'' in each case, 
we conducted pairwise t-test at significance level of $p< 0.05$.
The win/tie/loss counts in all cases are reported in Table~\ref{tablest}.
We can see that overall the proposed PML-GAN significantly outperforms 
PML-NI, PARTICLE-VLS, PARTICLE-MAP, PML-LC, PML-FP, and ML-RBF 
in 80.6\%, 75\%, 77\%, 81.1\%, 82.6\%, and 90.8\% of the cases respectively.
This again validates the efficacy of the proposed method for PML.

\textbf{Impact of Irrelevant Labels}
To demonstrate how would the number of irrelevant labels affect the performance of PML methods, 
we plotted the experimental results on the {\em delicious} dataset with different average numbers of candidate labels
in Figure~\ref{imgonenron}.
We can see with the increase of the number of irrelevant labels, consequently the average number of candidate labels,
the performance of each method in general degrades. 
Nevertheless, the proposed PML-GAN consistently outperforms all the other methods. 
Moreover, in terms of Hamming loss, the performance of PML-GAN actually is quite stable with the increase of the noisy labels.
This validates the effectiveness of PML-GAN in irrelevant noisy label disambiguation.
%

\begin{figure*}[t]
\centering
\begin{minipage}[b]{1\textwidth}
\centering	
\includegraphics[width=.95\textwidth]{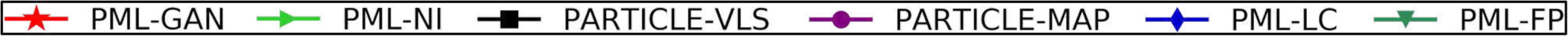} \\
\end{minipage}
\subfigure[Hamming Loss]{
\label{fig:example1}
\includegraphics[width=3.0cm,height=2.0cm]{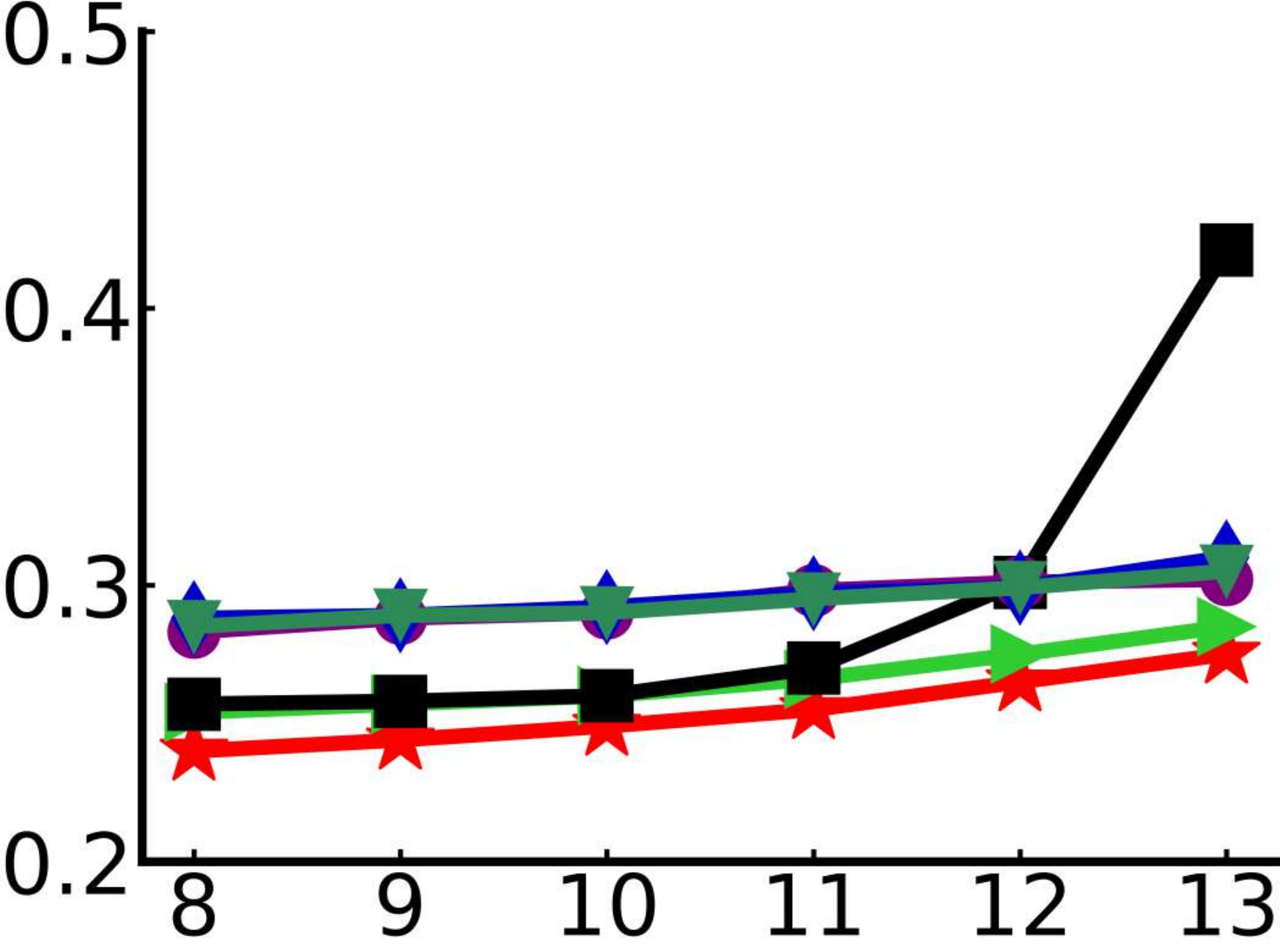}
}
\hspace{.08in}
\subfigure[Ranking Loss]{
\label{fig:example1}
\includegraphics[width=3.0cm,height=2.0cm]{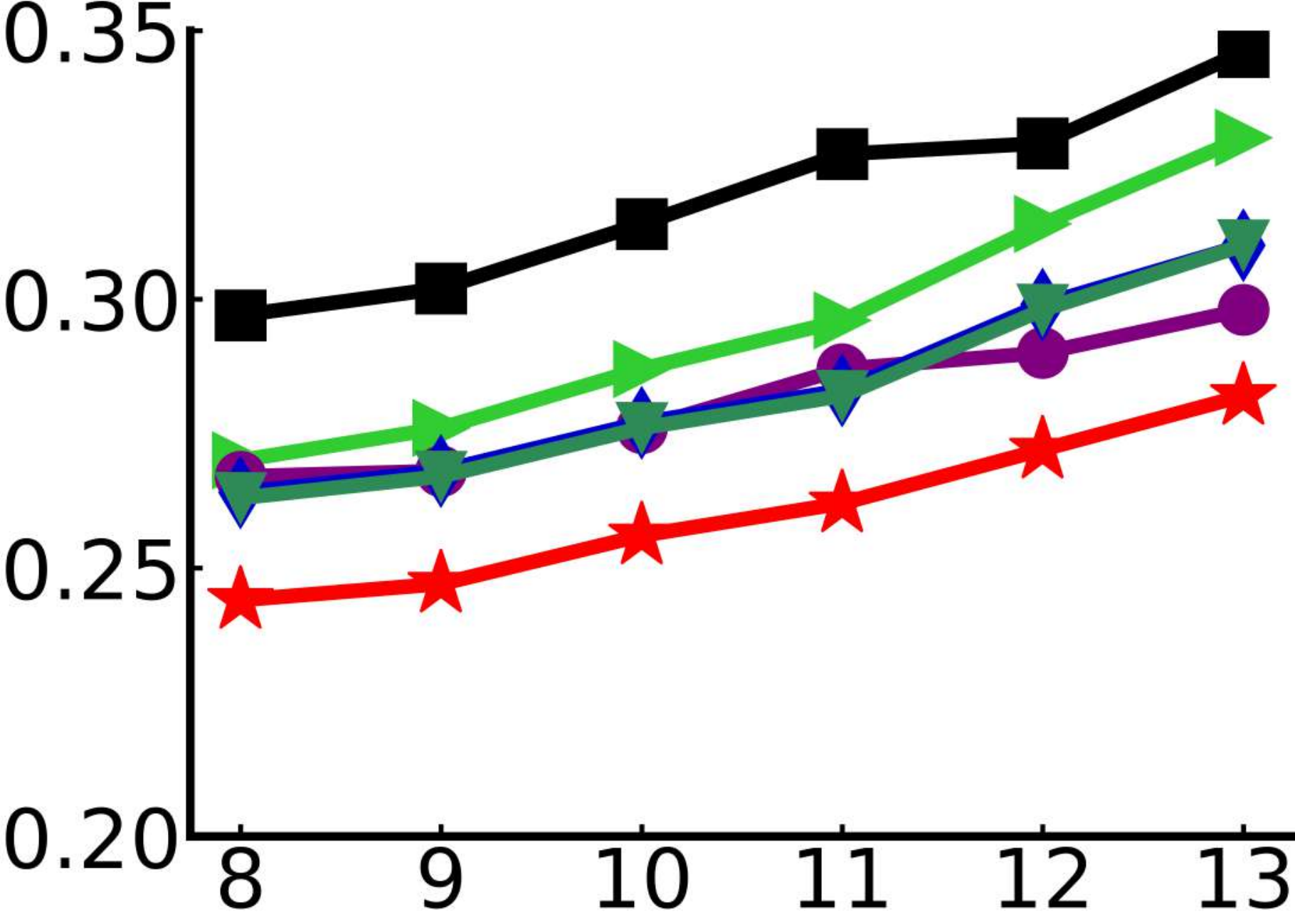}
}
\hspace{.08in}
\subfigure[One Error]{
\label{fig:example1}
\includegraphics[width=3.0cm,height=2.0cm]{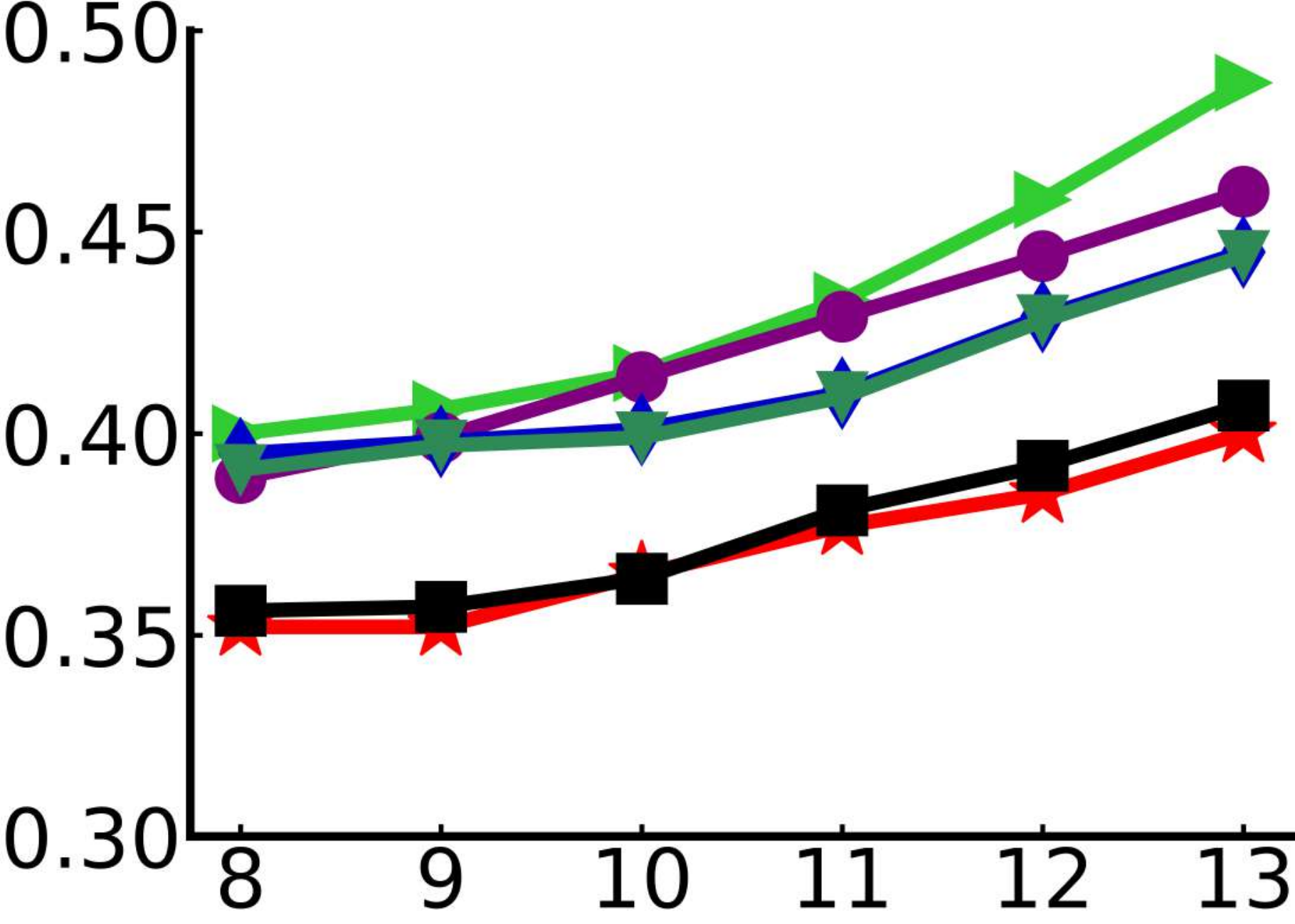}
}
\hspace{.08in}
\subfigure[Average Precision]{
\label{fig:example}
\includegraphics[width=3.0cm,height=2.0cm]{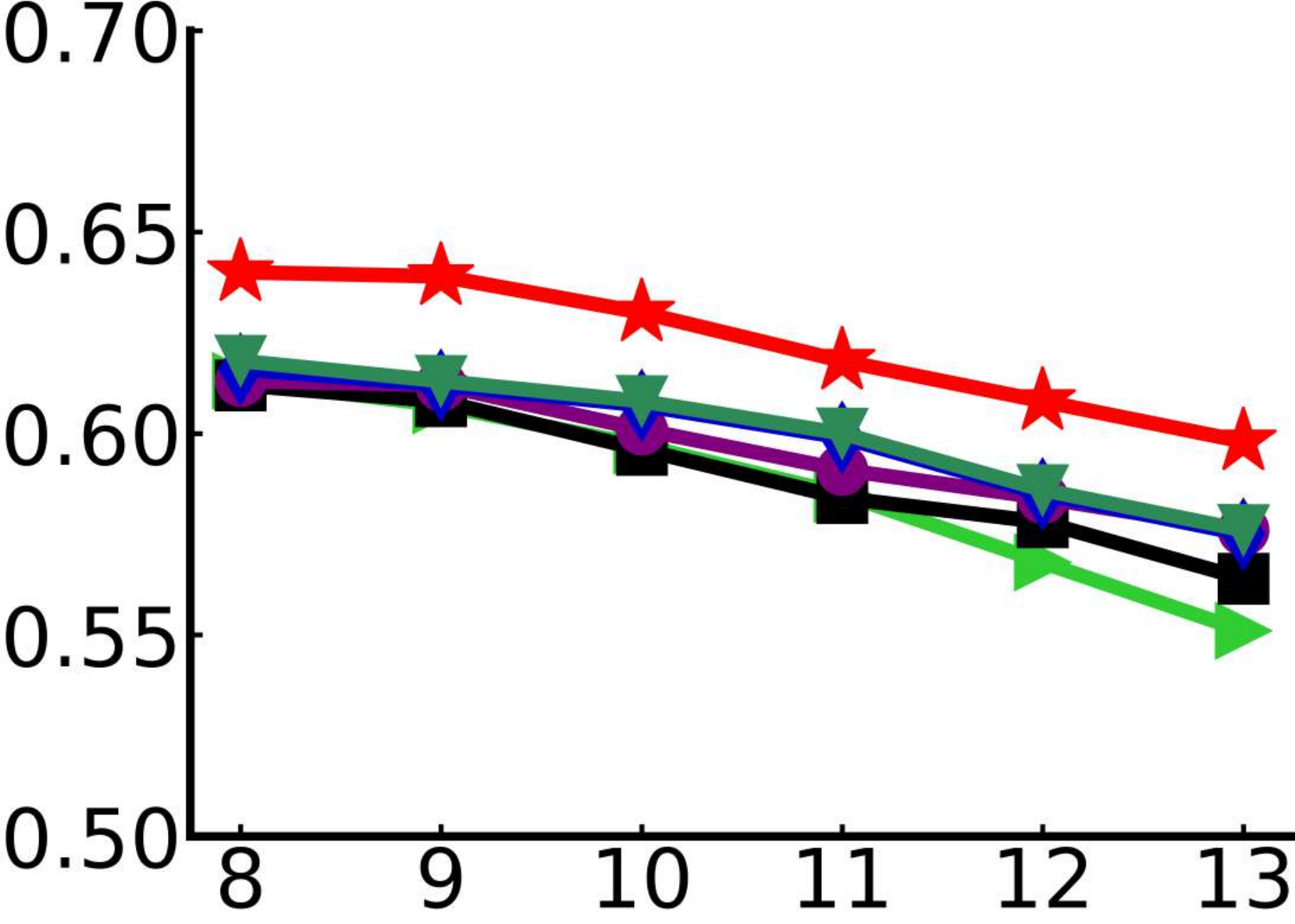}
}
\vskip -.1in
\caption{Test results with different average numbers of candidate labels on the delicious dataset.}
\label{imgonenron}
\vskip -.1in
\end{figure*}

\begin{table*}[t]
	\centering
	\caption{Win/tie/loss counts of pairwise t-test (with $p < 0.05$ ) between \textsc{PML-GAN} and each comparison method
	over all dataset variants with different numbers of candidate labels.}
	\vskip .1in
	\label{tablest}
	\setlength{\tabcolsep}{4pt} 
   \begin{threeparttable}
	\resizebox{\textwidth}{!}{%
		\begin{tabular}{p{3.3cm}<{}|p{1.5cm}<{\centering}p{2.5cm}<{\centering}p{2.5cm}<{\centering}p{1.5cm}<{\centering}
p{1.5cm}<{\centering}p{1.5cm}<{\centering}}
			\hline
          \multirow{2}{*}{Evaluation Metric} & \multicolumn{5}{c}{\textsc{PML-GAN} \hspace{0.1cm} vs \hspace{0.1cm} --} \\
                                               &\textsc{Pml-Ni} &\textsc{Particle-Vls} &\textsc{Particle-Map} &\textsc{Pml-Lc} &\textsc{Pml-fp}   &\textsc{Ml-rbf}  \\ \hline
                              Hamming loss        &36/11/2  &39/6/4     &38/9/2       &40/7/2       &40/3/6       &45/4/0\\
                              Ranking loss        &38/11/0   &38/9/2     &38/8/3       &38/8/3       &40/5/4       &44/3/2\\
                              One error           &44/5/0   &33/12/4    &39/7/3       &41/8/0       &40/9/0       &43/6/0\\
                              Average precision   &40/8/1   &37/8/4     &36/10/3      &40/6/3       &42/4/3       &46/3/0\\
            \hline
                              Total               &158/35/3 &147/35/14  &151/34/11    &159/29/8    &162/21/13    &178/16/2  \\
            \hline
		\end{tabular}%
	}
 \end{threeparttable}
\vskip -.1in	
\end{table*}

\begin{table*}[t!]
	\centering
	\caption{Comparison results of PML-GAN and its three ablation variants. 
	}
	\vskip .1in
	\label{cstable}
	\setlength{\tabcolsep}{4.5pt} 
	\resizebox{\textwidth}{!}{%
		\begin{tabular}{lcccc|cccc}
			\hline 
			\mline{Data set} &\textsc{Pml-Gan} &\textsc{Cls-Gen} &\textsc{Cls-Gan} &\textsc{Cls-ml} &\textsc{Pml-Gan} &\textsc{Cls-gen}  &\textsc{Cls-gan} &\textsc{Cls-ml} \\
			\hline
            \multicolumn{5}{l}{Hamming loss (the smaller, the better)} &\multicolumn{4}{l}{Ranking loss (the smaller, the better)}\\  
            \hline
            music\_emotion  &\textbf{.200$\pm$.004} &.203$\pm$.003	         &.202$\pm$.004         	 &.207$\pm$.004
&\textbf{.242$\pm$.007} &.249$\pm$.007	         &.244$\pm$.007         	 &.250$\pm$.010 \\
            music\_style    &\textbf{.115$\pm$.002} &.118$\pm$.003	         &.117$\pm$.004         	 &.121$\pm$.001
&\textbf{.145$\pm$.006} &.147$\pm$.007	         &.149$\pm$.007         	 &.155$\pm$.004 \\
            mirflickr       &\textbf{.170$\pm$.003} &.173$\pm$.004	         &.174$\pm$.005         	 &.177$\pm$.004
&\textbf{.124$\pm$.014} &.131$\pm$.019	         &.133$\pm$.021         	 &.136$\pm$.019 \\ \hline
            image           &\textbf{.202$\pm$.006} &.206$\pm$.005	         &.204$\pm$.008         	 &.220$\pm$.006
&\textbf{.191$\pm$.010} &.195$\pm$.010	         &.196$\pm$.016         	 &.201$\pm$.010 \\
            scene           &\textbf{.132$\pm$.007} &.140$\pm$.010	         &.138$\pm$.005         	 &.148$\pm$.008
&\textbf{.123$\pm$.009} &.130$\pm$.007	         &.137$\pm$.010         	 &.140$\pm$.020 \\  \hline
            yeast           &\textbf{.213$\pm$.008} &.219$\pm$.007	         &.216$\pm$.003         	 &.222$\pm$.006
&\textbf{.194$\pm$.011} &.199$\pm$.008	         &.195$\pm$.005         	 &.203$\pm$.006 \\
            enron           &\textbf{.186$\pm$.003} &.273$\pm$.015	         &.277$\pm$.014         	 &.281$\pm$.012
&\textbf{.182$\pm$.012} &.185$\pm$.009	         &.188$\pm$.009         	 &.189$\pm$.009 \\
            corel5k         &\textbf{.118$\pm$.001} &\textbf{.118$\pm$.001}	 &.120$\pm$.003         	 &.122$\pm$.003
&\textbf{.295$\pm$.011} &.304$\pm$.013	         &.299$\pm$.007         	 &.306$\pm$.015 \\
            eurlex\_dc   	&\textbf{.044$\pm$.001} &.050$\pm$.001	         &.047$\pm$.001         	 &.054$\pm$.001
&\textbf{.067$\pm$.005} &.068$\pm$.004	         &.068$\pm$.005         	 &.071$\pm$.008 \\
            eurlex\_sm  	&\textbf{.083$\pm$.002} &.085$\pm$.001	         &.086$\pm$.001         	 &.088$\pm$.002
&\textbf{.122$\pm$.007} &.125$\pm$.003	         &.125$\pm$.005         	 &.127$\pm$.004 \\
            delicious   	&\textbf{.249$\pm$.002} &.251$\pm$.003	         &.252$\pm$.001         	 &.255$\pm$.002
&\textbf{.258$\pm$.004} &.261$\pm$.007	         &.259$\pm$.005         	 &.269$\pm$.003 \\
            tmc2007      	&\textbf{.084$\pm$.001} &.086$\pm$.001	         &.086$\pm$.002              &.091$\pm$.001
&\textbf{.070$\pm$.001} &.073$\pm$.002	         &.072$\pm$.001         	 &.075$\pm$.003 \\
			\hline
            \multicolumn{5}{l}{Average precision  (the larger, the better)} &\multicolumn{4}{l}{One error (the smaller, the better)}\\
            \hline
            music\_emotion  &\textbf{.621$\pm$.013} &.608$\pm$.013	         &.612$\pm$.014         	 &.605$\pm$.013
&\textbf{.450$\pm$.028} &.465$\pm$.019	         &.469$\pm$.029         	 &.478$\pm$.028 \\
            music\_style    &\textbf{.732$\pm$.010} &.725$\pm$.012	         &.726$\pm$.011         	 &.720$\pm$.004
&\textbf{.347$\pm$.016} &.359$\pm$.021	         &.356$\pm$.018         	 &.367$\pm$.007 \\
            mirflickr       &\textbf{.777$\pm$.027} &.765$\pm$.037	         &.761$\pm$.036         	 &.754$\pm$.036
&\textbf{.336$\pm$.059} &.384$\pm$.084	         &.399$\pm$.067         	 &.417$\pm$.086 \\ \hline
            image           &\textbf{.775$\pm$.010} &.766$\pm$.009	         &.766$\pm$.018         	 &.758$\pm$.011
&\textbf{.342$\pm$.014} &.359$\pm$.016	         &.359$\pm$.029         	 &.364$\pm$.021 \\
            scene           &\textbf{.801$\pm$.012} &.793$\pm$.009	         &.783$\pm$.013         	 &.780$\pm$.021
&\textbf{.321$\pm$.022} &.330$\pm$.017	         &.349$\pm$.020         	 &.350$\pm$.027 \\  \hline
            yeast           &\textbf{.732$\pm$.014} &.723$\pm$.012	         &.730$\pm$.008         	 &.715$\pm$.009
&\textbf{.245$\pm$.017} &.262$\pm$.025	         &.247$\pm$.012         	 &.264$\pm$.018 \\
            enron           &\textbf{.665$\pm$.019} &.658$\pm$.017	         &.648$\pm$.028         	 &.645$\pm$.021
&\textbf{.307$\pm$.035} &.328$\pm$.024	         &.347$\pm$.049         	 &.350$\pm$.035 \\
            corel5k         &\textbf{.441$\pm$.012} &.431$\pm$.015	         &.439$\pm$.009         	 &.428$\pm$.017
&\textbf{.685$\pm$.015} &.705$\pm$.023	         &.690$\pm$.018         	 &.707$\pm$.020 \\
            eurlex\_dc   	&\textbf{.797$\pm$.009} &.790$\pm$.009	         &.792$\pm$.009              &.779$\pm$.015
&\textbf{.307$\pm$.013} &.310$\pm$.015	         &.312$\pm$.015         	 &.315$\pm$.021 \\
            eurlex\_sm  	&\textbf{.720$\pm$.009} &.713$\pm$.005	         &.713$\pm$.008         	 &.711$\pm$.005
&\textbf{.339$\pm$.013} &.351$\pm$.009	         &.350$\pm$.011         	 &.356$\pm$.010 \\
            delicious   &\textbf{.630$\pm$.006} &.627$\pm$.009	             &.626$\pm$.005              &.620$\pm$.002
&\textbf{.369$\pm$.009} &.375$\pm$.015	         &.381$\pm$.012         	 &.386$\pm$.003 \\
            tmc2007      	&\textbf{.821$\pm$.002} &.817$\pm$.003	         &.818$\pm$.004              &.815$\pm$.004
&\textbf{.202$\pm$.007} &.205$\pm$.004	         &.206$\pm$.007         	 &.210$\pm$.007 \\
			\hline
		\end{tabular}%
	}
\vskip -.1in	
\end{table*}

\subsection{Ablation Study}

As shown in Eq.(\ref{eq:objective}), 
the objective of PML-GAN contains three parts: classification loss, generation loss and adversarial loss.
The generation loss and adversarial loss are integrated 
to assist the predictor training.
To investigate and validate the contribution of the generation loss and adversarial loss, 
we conducted an ablation study by comparing PML-GAN with three of its ablation variants:
(1) CLS-GEN, which drops the adversarial loss; 
(2) CLS-GAN, which drops the generation loss; 
and (3) CLS-ML, which only uses the classification loss by dropping both the adversarial loss and generation loss.
The comparison results are reported in Table \ref{cstable}. 
We can see that comparing to the full model, all three variants produced inferior results in general.
Among the three variants, both {\em CLS-GEN} and {\em CLS-GAN} outperform {\em CLS-ML} in most cases. 
This suggests that both generation loss and adversarial loss are critical terms for the proposed model.
Moreover, even the baseline variant {\em CLS-ML} still produces some reasonable PML results.
This suggests the integration of our proposed prediction network and disambiguation network is also effective.

\section{Conclusion}

In this paper, we proposed a novel adversarial model for PML. 
The proposed model comprises four component networks, 
which form an encoder-decoder framework to improve noise label disambiguation
and boost multi-label learning performance.
The training problem forms a minimax adversarial optimization,
which is solved using an alternative min-max procedure with minibatch stochastic gradient descent. 
We conducted extensive experiments on multiple PML datasets.
The results show that the proposed model achieves
the state-of-the-art PML performance.

\bibliography{paperbib}
\bibliographystyle{abbrv}

\end{document}